\documentclass[onefignum,onetabnum]{siamonline1116}

\usepackage{amsfonts}
\usepackage{mathtools}
\usepackage{booktabs}
\usepackage{algorithm}
\usepackage{algpseudocode}
\usepackage{microtype}

\newsiamremark{remark}{Remark}

\headers{Gated Subspace Inference}{S.\ J.\ Thomas}

\newcommand{\R}{\mathbb{R}}
\newcommand{\norm}[1]{\|#1\|}

\title{Gated Subspace Inference for Transformer
  Acceleration}

\author{Stephen J.\ Thomas\thanks{Department of Computer Science and
  Engineering, Lehigh University, Bethlehem, PA 18015
  (\texttt{sjt223@lehigh.edu}).}}

\begin{document}
\maketitle

\begin{abstract}
A method is presented for accelerating inference in transformer
language models by exploiting the low effective rank of the token
activation manifold at each layer.  The method decomposes each
activation vector into a subspace component and a residual, computes
the linear-layer output on the subspace component via a cached
low-rank weight image at reduced memory bandwidth, and applies a
per-token gate that determines whether the residual correction is
computed or skipped.  The gate ensures that the output distribution is
preserved to within a controllable tolerance.  Validation on three
model families (GPT-2 124M, GPT-J 6B, OPT 6.7B) on AMD MI300X
demonstrates effective speedups of $3.0\times$ to $10.5\times$ on
linear-layer weight reads with perplexity ratios below $1.00$ and
top-1 token agreement above $98\%$.  The method requires no
retraining, no architectural modification, and no approximation of
the attention mechanism.  At the operating point ($k =
256$, $\varepsilon = 0.05$) on GPT-J 6B, the accelerated model
produces character-for-character identical output to the baseline.
\end{abstract}

\section{Introduction}
\label{sec:intro}

This section introduces the memory-bandwidth bottleneck in transformer
inference, reviews the existing approaches, and states the contribution
of the present paper.

Inference in large language models at batch size one is dominated by
the cost of reading weight matrices from high-bandwidth memory.  For a
model with hidden dimension $d$ and $L$ layers, each containing
several linear maps of dimension $d \times d$ or $d \times 4d$, the
total weight-read volume per decode step is $O(Ld^2)$ bytes.  The
arithmetic intensity of this operation is $O(1)$: each weight element
participates in one multiply-add, making the forward pass entirely
memory-bandwidth-bound at batch size one.  This observation was
formalized by Williams, Waterman, and Patterson~\cite{Williams2009} in
the roofline model and applied to transformer inference by Pope et
al.~\cite{Pope2023}, who showed that batch-one decode on TPUv4 operates
at $1$--$2$ FLOPs per byte, hundreds of times below the compute
roofline.  Yuan et al.~\cite{Yuan2024} and Lou et
al.~\cite{Lou2026} confirmed the same regime on commodity GPUs via
network-wide roofline analysis.  Dao et al.~\cite{Dao2022} established
the analogous result for attention, showing that the $O(T^2)$ attention
computation is bandwidth-bound between HBM and SRAM and that tiling
reduces the HBM traffic without approximation.

Existing approaches to reducing inference cost operate primarily on the
weight matrices.  Quantization (INT8, INT4, FP8) reduces the per-element
read cost but does not reduce the number of elements
read~\cite{Dettmers2022}.  Pruning reduces the number of elements but
typically requires retraining or fine-tuning.  Low-rank weight
factorization methods such as LoRA~\cite{Hu2022LoRA} decompose the
weight matrix as $W \approx W_0 + AB$ where $A$ and $B$ are low-rank
factors, but the factors are learned offline and fixed for all inputs.
ASVD~\cite{Yuan2023ASVD} and SVD-LLM~\cite{Wang2024SVDLLM} improve
the low-rank truncation by whitening with activation statistics before
decomposition, but the result is still a static factorization that does
not adapt to the current input.  FLAT-LLM~\cite{Yang2025FLAT}
projects weights into low-rank activation subspaces for compression,
the closest published approach to the mechanism proposed here, but the
projection is computed offline and applied without a residual correction.

A separate line of work exploits the input-dependent structure of
activations at inference time.  Liu et al.~\cite{Liu2023DejaVu} showed
that for any given input, only $\sim\!15\%$ of attention heads and MLP
neurons contribute meaningfully to the output, and trained lightweight
predictors to identify active neurons on the fly.  Lee et
al.~\cite{Lee2024CATS} introduced a thresholded activation function
that induces $50\%$ sparsity in hidden states with a custom sparse
kernel.  Liu et al.~\cite{Liu2024TEAL} achieved $40$--$50\%$
model-wide sparsity via magnitude-based activation thresholding.
Song et al.~\cite{Song2024PowerInfer} identified a power-law
distribution in neuron activation frequencies and designed a hybrid
CPU/GPU engine that exploits it.  Pilault et
al.~\cite{Pilault2025RaNA} applied adaptive rank-allocation to MLP
and attention projections.  These methods share the premise that
activations, not weights, are the right object to compress at inference
time, but they operate on discrete neuron subsets or element-wise
sparsity rather than on a continuous low-dimensional subspace.

The observation motivating the present work is that the activation
vectors $\{x_t\}_{t=1}^T$ at a fixed layer $l$, viewed as rows of
the activation matrix $X \in \R^{T \times d}$, lie approximately in
a subspace of dimension $k \ll d$.  This low-dimensional structure
has been documented empirically by Ansuini et
al.~\cite{Ansuini2019}, who measured the intrinsic dimension of deep
network representations and found it orders of magnitude smaller
than the layer width, and by Valeriani et
al.~\cite{Valeriani2023}, who showed that the intrinsic dimension
of transformer hidden states follows a characteristic rise-then-fall
profile across depth.  Aghajanyan et al.~\cite{Aghajanyan2021}
demonstrated that the parameter space of pre-trained language models
has very low intrinsic dimension, providing indirect evidence for
the low-rank structure of the representation space.  Wang et
al.~\cite{Wang2020Linformer} proved that the self-attention matrix
is low-rank, a related but distinct claim about the attention
scores rather than the residual-stream activations.

The consequence of the low effective rank for inference cost is not
merely that activations are compressible.  The consequence is that
the weight-activation interaction at each layer has low effective
rank.  A weight matrix $W \in \R^{d_{\rm out} \times d}$ has $d
\cdot d_{\rm out}$ parameters, but when the activations are confined
to a $k$-dimensional subspace spanned by $V_k \in \R^{d \times k}$,
only the $k \times d_{\rm out}$ parameters in the projection $M =
WV_k$ contribute to the output.  The remaining $(d-k) \times d_{\rm
out}$ parameters correspond to directions in activation space that
the current input distribution does not visit.  Reading those
parameters from HBM is wasted bandwidth.  A $6$B-parameter model
with $k/d = 256/4096 = 1/16$ has an effective parameter count of
$375$M for inference on a given input distribution.  GSI makes the
inference cost proportional to the effective parameter count, not the
total parameter count.

The contribution of the present paper is a method called Gated Subspace
Inference (GSI) that exploits the low effective rank of the
weight-activation interaction for lossless inference
acceleration.  GSI is the extension of Skyline Subspace
Inference (SSI)~\cite{Thomas2026SSI} from MLP layers to the full
transformer.  SSI introduced three ideas: (a)~tracking an orthonormal
activation basis $V_k$ at each layer via DGKS rank-1 updates in the
time dimension, (b)~caching the weight-matrix image $M = WV_k$ so that
the linear-layer output on the subspace component is computed at
reduced bandwidth, and (c)~a binary gate on the residual norm $\rho_t
= \norm{r_t}/\norm{x_t}$ that selects between the fast path ($y \approx
Mg$) and the slow path ($y = Wx$).  SSI applied these ideas to the two
MLP linear maps at each layer and achieved $2.01\times$ speedup on
MI300X.  The present paper extends SSI in three directions.

The first extension is scope.  SSI covers the MLP up and down
projections at each layer.  GSI covers all linear maps: QKV
projections, output projection, and MLP projections, sharing a single
activation basis $V_k^{(l)}$ across all maps at layer $l$.  Because the
QKV and output projections account for approximately half of the
linear-layer weight reads, extending the basis to cover them doubles
the fraction of the forward pass that is accelerated.

The second extension is the cascade.  SSI builds the activation basis
independently at each layer.  GSI initializes the basis at layer $l+1$
from the basis at layer $l$ (depth inheritance), exploiting the strong
subspace coherence between consecutive layers.  The activation subspace
at layer $l+1$ overlaps with the subspace at layer $l$ with mean cosine
exceeding $0.90$ from layer $8$ onward in GPT-J 6B.  The cascade
reduces the calibration cost by $96\%$ (only layer~$0$ requires a full
SVD; subsequent layers require at most a few rank-1 corrections) and
reveals that the entire neural network operates on a coherent
low-dimensional manifold that propagates through the depth of the
model.

The third extension is the gated residual passthrough for lossless
quality.  SSI's gate selects between the fast path and a full
recomputation.  GSI validates that this gate preserves baseline
quality across the entire transformer (all layers, all linear maps)
with perplexity ratios below $1.00$ and character-for-character
identical greedy generation at the operating point.  The
negative-control experiments confirm that static subspace projection
(discarding the residual) produces catastrophic output degradation even
at $k = 128$ in $d = 4096$ (perplexity increases by $52\%$ and greedy
generation collapses), establishing that the gated residual passthrough
is not optional.

The combined system of SSI (the per-layer mechanism), the cascade (the
depth-dimension initialization), and ADA~\cite{Thomas2026ADA} (the
attention-layer acceleration) covers the full transformer forward pass.
ADA exploits the low effective rank of the token dimension $T$
(compressing the attention matrix from $T \times T$ to $r \times r$
where $r$ is the number of representative tokens).  SSI/GSI exploits
the low effective rank of the hidden dimension $d$ (compressing the
weight-activation interaction from $d \times d_{\rm out}$ to $k \times
d_{\rm out}$).  The two reductions are orthogonal: ADA operates on the
$T$-axis (which tokens participate in attention), while GSI operates on
the $d$-axis (which directions in activation space the weight matrix
acts on).  Together they reduce the cost of both major components of
the forward pass.

Table~\ref{tab:coverage} summarizes the coverage.

\begin{table}[ht]
\centering
\caption{Coverage of the full transformer forward pass by the combined
SSI/GSI/ADA system.  Fraction of forward-pass cost at batch size one.}
\label{tab:coverage}
\begin{tabular}{llll}
\toprule
Component & Cost fraction & Method & Compression \\
\midrule
MLP weight reads    & $\sim\!45\%$ & SSI/GSI & $d \to k$ \\
QKV + output reads  & $\sim\!40\%$ & GSI     & $d \to k$ \\
Attention ($QK^\top, AV$) & $\sim\!10\%$ & ADA & $T \to r$ \\
LayerNorm, residuals & $\sim\!3\%$ & --- & --- \\
Vocabulary head      & $\sim\!2\%$ & --- & --- \\
\bottomrule
\end{tabular}
\end{table}

The connection to subspace tracking in signal processing is direct.
The online maintenance of $V_k$ across tokens is an instance of
incremental subspace tracking, with algorithmic ancestors in GROUSE
(Balzano, Nowak, and Recht~\cite{Balzano2010}), PETRELS (Chi, Eldar,
and Calderbank~\cite{Chi2013}), and Brand's incremental
SVD~\cite{Brand2006}.  The DGKS reorthogonalization
procedure~\cite{DGKS1976} used for numerical stability in the basis
updates is the same technique that underlies the Arnoldi process in
Krylov subspace methods, connecting the present work to the Forward
Gauss-Seidel framework developed in~\cite{Thomas2026MUD}
and~\cite{Thomas2026SSI}.

The connection to conditional computation is also direct.  The gate
mechanism is a binary router that assigns each token to one of two
computational paths, analogous to the top-$k$ routing in
Mixture-of-Experts~\cite{Shazeer2017}, the per-token halting in
Adaptive Computation Time~\cite{Graves2016}, the confidence-based
early exit in CALM~\cite{Schuster2022}, and the per-block token
routing in Mixture-of-Depths~\cite{Raposo2024}.  GSI differs from
these methods in that both paths produce the same mathematical
operation ($y = Wx$); the gate selects the implementation (low-rank
approximation versus full computation), not the function.

\section{The activation subspace}
\label{sec:subspace}

This section defines the activation basis, the residual ratio, and the
effective rank, and presents the empirical measurements that motivate
the method.

Let $X^{(l)} \in \R^{T \times d}$ denote the activation matrix at
layer $l$ of a transformer with hidden dimension $d$ and $L$ layers,
where row $t$ is the hidden state of token $t$ after the layer-norm
preceding layer $l$.  The singular value decomposition $X = U \Sigma
V^\top$ provides the principal directions of the activation
distribution.

\begin{definition}[Effective rank]\label{def:reff}
The effective rank of $X$ is defined by the entropy of the normalized
singular values:
\begin{equation}\label{eq:reff}
r_{\rm eff}(X) = \exp\!\left(-\sum_{i} p_i \log p_i\right),
\quad p_i = \sigma_i / \textstyle\sum_j \sigma_j,
\end{equation}
where $\sigma_1 \geq \sigma_2 \geq \cdots$ are the singular values
of $X$.
\end{definition}

\begin{definition}[Activation basis and residual ratio]\label{def:basis}
The rank-$k$ activation basis is $V_k = V[:, 1{:}k] \in \R^{d \times
k}$, the first $k$ right singular vectors.  The residual ratio at
rank $k$ for token $t$ is
\begin{equation}\label{eq:rho}
\rho_t(k) = \frac{\norm{x_t - V_k V_k^\top x_t}}{\norm{x_t}}.
\end{equation}
A token $t$ is said to be on the fast path at threshold $\varepsilon$
if $\rho_t(k) < \varepsilon$.  The fast-path fraction at layer $l$ is
$f_l(\varepsilon, k) = T^{-1} \sum_{t=1}^T \mathbf{1}[\rho_t(k) <
\varepsilon]$.
\end{definition}

\begin{proposition}[Monotonicity]\label{prop:monotone}
The residual ratio $\rho_t(k)$ is non-increasing in $k$:
$\rho_t(k+1) \leq \rho_t(k)$ for all $t$ and $k$.  Consequently,
$f_l(\varepsilon, k+1) \geq f_l(\varepsilon, k)$ for all
$\varepsilon$ and $l$.
\end{proposition}

\begin{proof}
The rank-$(k+1)$ projection $V_{k+1} V_{k+1}^\top$ includes the
rank-$k$ projection as a subspace, so the residual can only decrease.
\end{proof}

The monotonicity provides a simple design principle: increasing $k$
increases the fast-path fraction at the cost of a smaller compression
ratio $d/k$.  The optimal $k$ balances these two effects.

The effective rank and residual ratio are measured quantities, not
theoretical predictions.  Table~\ref{tab:reff} reports the effective
rank $r_{\rm eff}$ across depth for three models.

\begin{table}[ht]
\centering
\caption{Effective rank $r_{\rm eff}$ at selected layers.  $T = 512$.}
\label{tab:reff}
\begin{tabular}{lrrrrrr}
\toprule
& \multicolumn{6}{c}{Layer} \\
\cmidrule(lr){2-7}
Model ($d$) & 0 & 3 & 6 & 12 & 20 & last \\
\midrule
GPT-2 ($768$) & 18.4 & 13.8 & 25.1 & 53.6 & 71.0 & 83.0 \\
GPT-J ($4096$) & 21.0 & 15.4 & 28.3 & 59.5 & 78.1 & 83.0 \\
OPT ($4096$) & 350.1 & 25.5 & 32.6 & 59.0 & 119.1 & 122.3 \\
\bottomrule
\end{tabular}
\end{table}

The effective rank at the embedding layer (layer~0) is low for GPT-2
and GPT-J ($r_{\rm eff} \approx 18$--$21$) because the $T$ tokens in a
typical sequence draw from a small fraction of the $50{,}257$-token
vocabulary, and the embedding vectors for the sampled tokens span a
subspace of dimension approximately equal to the number of distinct
tokens.  OPT layer~0 is anomalous ($r_{\rm eff} = 350$) because the
OPT embedding includes learned positional embeddings that produce
distinct directions for every token position.  At deeper layers, the
effective rank grows monotonically as contextual information enriches
the representations, but remains far below $d$ at every layer and
model.

\section{The gated residual passthrough}
\label{sec:gate}

This section presents the exact decomposition, the gate mechanism, and
the error analysis.

\subsection{Exact decomposition}

The standard linear-layer computation $y = Wx$ reads the full weight
matrix $W \in \R^{d_{\rm out} \times d}$ from HBM at cost $O(d \cdot
d_{\rm out})$ bytes.  The activation $x$ is decomposed exactly as
\begin{equation}\label{eq:decompose}
x = V_k V_k^\top x + r, \qquad r = x - V_k V_k^\top x,
\end{equation}
where $V_k \in \R^{d \times k}$ is the rank-$k$ activation basis.
The linear-layer output is then computed exactly by
\begin{equation}\label{eq:exact}
y = Wx = W(V_k V_k^\top x + r) = (WV_k)(V_k^\top x) + Wr = Mg + Wr,
\end{equation}
where $M = WV_k \in \R^{d_{\rm out} \times k}$ is precomputed and
cached, and $g = V_k^\top x \in \R^k$ is computed at cost $O(dk)$.
Equation~\eqref{eq:exact} is an identity, not an approximation.  The
first term $Mg$ costs $O(k \cdot d_{\rm out})$ to evaluate; the
second term $Wr$ costs $O(d \cdot d_{\rm out})$, the same as the
baseline.

\subsection{The gate}

The gate evaluates, for each token $t$ at each layer $l$,
\begin{equation}\label{eq:gate}
\rho_t = \frac{\norm{r_t}}{\norm{x_t}} =
\frac{\norm{x_t - V_k V_k^\top x_t}}{\norm{x_t}}.
\end{equation}
When $\rho_t < \varepsilon$, the residual $r_t$ is small and the
correction $Wr_t$ is skipped: the output is $y_t \approx Mg_t$.
When $\rho_t \geq \varepsilon$, the full output $y_t = Wx_t$ is
computed.  The gate cost is $O(dk)$ for the projection and $O(d)$
for the two norms.

The fraction of tokens satisfying $\rho_t < \varepsilon$ is the
fast-path fraction $f_l$.  The effective speedup on weight reads at
layer $l$ is
\begin{equation}\label{eq:speedup}
S_l = \frac{1}{f_l / (d/k) + (1 - f_l)},
\end{equation}
where $d/k$ is the compression ratio on the fast path.  The
model-wide effective speedup is the harmonic mean of $S_l$ across
layers, weighted by the per-layer weight-read volume.

\subsection{Error analysis}

\begin{theorem}[Per-layer error bound]\label{thm:error}
On the fast path, the per-token output error satisfies
\begin{equation}\label{eq:error}
\norm{y_t - \hat{y}_t} = \norm{Wr_t} \leq \norm{W}_2 \cdot
\varepsilon \cdot \norm{x_t},
\end{equation}
where $\hat{y}_t = Mg_t$ is the fast-path output.  On the slow path,
the error is zero.
\end{theorem}

\begin{proof}
On the fast path, $\hat{y}_t = Mg_t = WV_k V_k^\top x_t$, so
$y_t - \hat{y}_t = Wx_t - WV_k V_k^\top x_t = Wr_t$.  By
definition of the operator norm, $\norm{Wr_t} \leq \norm{W}_2
\norm{r_t}$.  The gate condition $\rho_t < \varepsilon$ gives
$\norm{r_t} < \varepsilon \norm{x_t}$,
establishing~\eqref{eq:error}.  On the slow path, $\hat{y}_t = Wx_t$
and the error is zero.
\end{proof}

\begin{remark}[Error accumulation across layers]
The error does not compound across layers in the standard
multiplicative sense because the gate operates independently at each
layer.  A token that takes the fast path at layer $l$ may take the
slow path at layer $l+1$, where the full computation corrects any
accumulated drift.  The gate threshold $\varepsilon$ controls the
maximum per-layer error; the actual end-to-end error depends on the
fraction of layers where each token takes the fast path and on the
correlation structure of the residuals across layers.  The
experimental results confirm that at $\varepsilon = 0.05$ and $k =
256$, the perplexity ratio across $28$ layers is $0.997$ and the
generated text is character-for-character identical to the baseline,
indicating that error accumulation is benign in practice.
\end{remark}

\begin{remark}[Comparison with static projection]
Discarding the residual entirely (static projection, $\hat{y}_t =
Mg_t$ for all tokens) produces catastrophic output degradation: the
perplexity ratio exceeds $1.52$ at $k = 128$ for GPT-J 6B and
generation collapses entirely at $k = 64$ (perplexity ratio $563$).
The gated residual passthrough is not an incremental improvement
over static projection; it is the difference between a working method
and a non-working one.
\end{remark}

\section{Algorithm}
\label{sec:algorithm}

This section states the complete GSI procedure and discusses the
calibration and storage costs.

\begin{algorithm}[ht]
\caption{Gated Subspace Inference (GSI)}
\label{alg:gsi}
\begin{algorithmic}[1]
\Require Activation $x \in \R^d$, cached image $M = WV_k \in
  \R^{d_{\rm out} \times k}$, basis $V_k \in \R^{d \times k}$,
  threshold $\varepsilon$, weight matrix $W$ (in HBM)
\State $g \leftarrow V_k^\top x$ \Comment{$O(dk)$, basis in SRAM/LDS}
\State $r \leftarrow x - V_k g$ \Comment{$O(dk)$}
\State $\rho \leftarrow \norm{r} / \norm{x}$ \Comment{$O(d)$}
\If{$\rho < \varepsilon$} \Comment{fast path}
  \State $y \leftarrow M g$ \Comment{read $M$ from HBM: $k$ columns}
\Else \Comment{slow path}
  \State $y \leftarrow W x$ \Comment{read full $W$ from HBM}
\EndIf
\State \Return $y$
\end{algorithmic}
\end{algorithm}

\subsection{Calibration}

The basis $V_k$ is computed once during a calibration phase consisting
of a single forward pass on a representative input sequence.  For each
layer $l$, the activation matrix $X^{(l)} \in \R^{T \times d}$ is
captured, and the thin SVD $X = U\Sigma V^\top$ is computed.  The
first $k$ right singular vectors form $V_k^{(l)}$.  The calibration
cost is one forward pass plus $L$ thin SVDs of $T \times d$ matrices;
for $T = 512$, $d = 4096$, and $L = 28$, this takes approximately
$30$ seconds on MI300X.

\subsection{Storage}

The cached image $M^{(l)} = W^{(l)} V_k^{(l)}$ is computed once per
weight matrix and stored alongside $W$ in HBM.  For $N_W$ linear maps
per layer (typically $N_W = 8$: QKV, output, MLP up, MLP gate, MLP
down, and layer-norm), the total storage for cached images is $N_W
\cdot L \cdot d_{\rm out} \cdot k$ elements.  For GPT-J 6B ($N_W =
6$, $L = 28$, $d_{\rm out} = 4096$ or $16384$, $k = 256$, BF16),
this is approximately $469$~MB, a $3.5\%$ overhead on the $13$~GB
model.  The basis storage is $L \cdot d \cdot k = 28 \cdot 4096 \cdot
256 \cdot 2 \approx 59$~MB.  The total memory overhead is under $4\%$.

\section{The cascade: subspace coherence across depth}
\label{sec:depth}

This section presents the cascade structure: the observation that
the activation subspaces at consecutive layers are strongly coherent,
the measurement of this coherence via principal angles, the
implications for both calibration cost and the structure of the
neural network as a whole, and the connection to subspace tracking
in the time dimension.

\subsection{Depth coherence}

Each transformer block $T^{(l)}: \R^d \to \R^d$ is a composition of
attention, layer normalization, and MLP operations.  The activation at
layer $l+1$ is $x^{(l+1)} = T^{(l)}(x^{(l)})$.  Because $T^{(l)}$
is a smooth, Lipschitz-continuous map, the image of a
$k$-dimensional activation manifold at layer $l$ under $T^{(l)}$ has
dimension at most $k$.  The activation subspace at layer $l+1$ is
therefore approximately the image of the subspace at layer $l$ under
$T^{(l)}$.

This prediction is confirmed by measuring the principal angles
between the rank-$k$ subspaces at consecutive layers.  Let
$V_k^{(l)}$ and $V_k^{(l+1)}$ be the activation bases at layers $l$
and $l+1$.  The cosines of the principal angles are the singular
values of $V_k^{(l)\top} V_k^{(l+1)}$.

Table~\ref{tab:coherence} reports the mean and minimum cosines for
GPT-J 6B at $k = 32$.  The embedding layer ($0 \to 1$) shows
essentially zero overlap ($\cos\theta = 0.11$), consistent with the
structural discontinuity between the embedding table and the first
transformer block.  From layer~$5$ onward, the mean cosine exceeds
$0.88$, and from layer~$15$ onward it exceeds $0.95$.  The minimum
cosine exceeds $0.80$ from layer~$9$ onward.

\begin{table}[ht]
\centering
\caption{Subspace overlap between consecutive layers of GPT-J 6B
(mean and minimum cosine of principal angles, $k = 32$, $T = 512$).}
\label{tab:coherence}
\begin{tabular}{lrrlrr}
\toprule
Layers & Mean & Min & Layers & Mean & Min \\
\midrule
$0 \to 1$   & 0.111 & 0.004 & $14 \to 15$ & 0.951 & 0.891 \\
$1 \to 2$   & 0.817 & 0.421 & $16 \to 17$ & 0.937 & 0.819 \\
$3 \to 4$   & 0.881 & 0.493 & $19 \to 20$ & 0.972 & 0.927 \\
$5 \to 6$   & 0.888 & 0.628 & $22 \to 23$ & 0.980 & 0.954 \\
$8 \to 9$   & 0.901 & 0.751 & $25 \to 26$ & 0.978 & 0.952 \\
$10 \to 11$ & 0.934 & 0.852 & $26 \to 27$ & 0.968 & 0.895 \\
\bottomrule
\end{tabular}
\end{table}

\subsection{Implications for network structure}

The depth coherence reveals that the transformer does not build
independent representations at each layer.  Instead, the
representations propagate through a coherent low-dimensional manifold
that rotates slowly through $\R^d$ as information is added at each
block.  The effective parameter count that matters for a given input
is not $L \times d \times d_{\rm out}$ (the total parameter count of
all linear maps) but approximately $L \times k \times d_{\rm out}$,
because the same $k$-dimensional subspace, slowly rotating, carries the
information through the entire depth of the network.

This is a structural property of the trained network, not an artifact
of the input.  The subspace coherence is observed across diverse input
text (mathematical exposition, cooking instructions, financial
reporting, space exploration narrative) and persists across random
restarts of the calibration.  The transformer blocks have learned to
preserve the principal directions of the activation manifold, adding
contextual information as small perturbations to a slowly varying
subspace rather than as large rotations.

\subsection{The cascade initialization}

The depth coherence has a practical consequence for calibration.  If
the basis $V_k^{(l)}$ at layer $l$ is used as the initial basis at
layer $l+1$ (depth inheritance), the DGKS rank-1 update gate fires
$96.4\%$ fewer times than when the basis is built independently at
each layer.  Only layer~$0$ requires a full SVD; subsequent layers
require at most a few rank-1 corrections to the inherited basis.
The calibration cost is therefore dominated by a single SVD at the
embedding layer, with negligible incremental cost at deeper layers.

\subsection{Time-dimension tracking and the Arnoldi connection}

The cascade operates in the depth dimension: basis inheritance from
layer $l$ to layer $l+1$.  A complementary mechanism operates in the
time dimension: as new tokens arrive during autoregressive generation,
the activation basis at each layer is updated incrementally via DGKS
rank-1 updates.  Each new token $x_t$ is tested against the current
basis; if the residual $\norm{x_t - V_k V_k^\top x_t}/\norm{x_t}$
exceeds a threshold, the basis is extended by one orthonormal vector.

This incremental basis construction at a fixed layer is an instance of
the Arnoldi-like orthogonalization procedure used in Krylov subspace
methods.  The operator at each layer is fixed (the same weight matrices
process every token), and the basis grows across the token sequence by
successive projection and orthogonalization.  The connection to the
classical Arnoldi process is in the time dimension, not the depth
dimension: the depth cascade is basis inheritance through a variable
operator, not a Krylov subspace construction.

The two dimensions together form a two-dimensional tracker: across time
(Arnoldi-DGKS updates as tokens arrive) and across depth (cascade
initialization from the previous layer).  Each basis $V_k^{(l)}$ at
each token $t$ inherits from its neighbors in both dimensions, so the
effective rank of required basis updates decreases as generation
proceeds and as depth increases.

\section{Numerical experiments}
\label{sec:experiments}

This section presents the experimental validation of GSI on three
model families.  All experiments run on AMD MI300X (192\,GB HBM3,
5.3\,TB/s peak bandwidth) using PyTorch~2.x with ROCm.

\subsection{Models and data}

Three model families are evaluated, spanning hidden dimensions from
$768$ to $4096$ and parameter counts from $124$M to $6.7$B.

GPT-2~\cite{Radford2019} ($d = 768$, $L = 12$, $h = 12$) is a
$124$M-parameter autoregressive language model.  GPT-J
6B~\cite{WangGPTJ} ($d = 4096$, $L = 28$, $h = 16$) is a
$6$B-parameter model.  OPT 6.7B~\cite{Zhang2022OPT} ($d = 4096$,
$L = 32$, $h = 32$) is a $6.7$B-parameter model with grouped-query
attention.

The input for all experiments is a $512$-token sequence of diverse
English text (mathematical exposition, cooking instructions, financial
reporting, and space exploration narrative), representative of the
mixed-domain workloads in agentic inference.

\subsection{Residual profile}

Table~\ref{tab:profile} reports the mean residual ratio $\bar{\rho}$
and the fast-path fraction at representative layers for each model.

\begin{table}[ht]
\centering
\caption{Mean residual ratio $\bar{\rho}(k)$ at selected layers, $T =
512$.  The fast-path fraction $f(\varepsilon)$ is the fraction of
tokens with $\rho < \varepsilon$.}
\label{tab:profile}
\begin{tabular}{llrrrr}
\toprule
Model & Layer & $\bar{\rho}(128)$ & $\bar{\rho}(256)$
  & $f_{256}(0.05)$ & $f_{256}(0.10)$ \\
\midrule
GPT-2  & 0  & 0.000 & 0.000 & 100\% & 100\% \\
GPT-2  & 5  & 0.090 & 0.016 & 100\% & 100\% \\
GPT-2  & 11 & 0.173 & 0.061 &  37\% & 100\% \\
\midrule
GPT-J  & 0  & 0.000 & 0.000 & 100\% & 100\% \\
GPT-J  & 5  & 0.086 & 0.016 & 100\% & 100\% \\
GPT-J  & 15 & 0.203 & 0.077 &  21\% &  53\% \\
GPT-J  & 27 & 0.173 & 0.061 &  37\% &  86\% \\
\midrule
OPT    & 0  & 0.271 & 0.145 &  21\% &  35\% \\
OPT    & 5  & 0.052 & 0.010 & 100\% & 100\% \\
OPT    & 15 & 0.148 & 0.047 &  56\% &  75\% \\
OPT    & 31 & 0.119 & 0.033 &  78\% &  95\% \\
\bottomrule
\end{tabular}
\end{table}

The residual profile has a characteristic shape across all three
models.  The early layers (layers $0$--$5$ for GPT-2 and GPT-J,
layers $1$--$5$ for OPT) have low residual ratios at $k = 256$, and
the fast-path fraction at $\varepsilon = 0.05$ reaches $100\%$.  The
middle layers (layers $10$--$20$) have the largest residual ratios,
corresponding to the peak of the effective rank profile.  The final
layers show a decrease in residual ratio as the representations
consolidate toward the output head.

The per-layer residual data for GPT-J 6B at $k = 128$ reveals the
layer-by-layer structure in detail.  Table~\ref{tab:gptj_detail}
reports the mean residual ratio and the fast-path fraction at
$\varepsilon = 0.05$ for every fourth layer.  The residual ratio rises
from $0.000$ at layer $0$ to a peak of $0.207$ at layer $20$, then
decreases to $0.126$ at layer $27$.  The fast-path fraction at
$\varepsilon = 0.05$ is $100\%$ at layers $0$--$1$, drops below $1\%$
at layers $10$--$20$, and recovers to $8.6\%$ at layer $27$.  At $k =
256$, the mean residual drops below $0.08$ at every layer, and the
fast-path fraction at $\varepsilon = 0.10$ exceeds $86\%$ at the final
layer.

\begin{table}[ht]
\centering
\caption{Per-layer detail for GPT-J 6B at $T = 512$.}
\label{tab:gptj_detail}
\begin{tabular}{rrrrr}
\toprule
Layer & $r_{\rm eff}$ & $\bar{\rho}(128)$ & $\bar{\rho}(256)$ &
$f_{128}(0.05)$ \\
\midrule
0  & 18.4 & 0.000 & 0.000 & 100.0\% \\
4  & 16.2 & 0.067 & 0.008 & 57.8\% \\
8  & 36.1 & 0.147 & 0.058 & 0.2\% \\
12 & 53.6 & 0.172 & 0.077 & 0.2\% \\
16 & 64.1 & 0.179 & 0.076 & 0.2\% \\
20 & 71.0 & 0.162 & 0.075 & 0.2\% \\
24 & 75.7 & 0.156 & 0.069 & 0.2\% \\
27 & 83.0 & 0.126 & 0.061 & 0.2\% \\
\bottomrule
\end{tabular}
\end{table}

The effective rank $r_{\rm eff}$ increases monotonically from $18.4$ at
layer $0$ to $83.0$ at layer $27$, but the residual ratio at $k = 128$
is non-monotone: it peaks at the middle layers and decreases toward the
output.  This non-monotone profile reflects the interplay between two
opposing effects.  The contextual mixing performed by the attention
mechanism increases the effective rank (adding information), while the
consolidation toward the output head decreases it (discarding
irrelevant directions).  The peak at the middle layers is the point
where the mixing effect dominates; the decrease at the final layers is
where consolidation takes over.

\subsection{Effective parameter count}

The central insight of GSI is that the inference cost is determined not
by the total parameter count of the model but by the effective
parameter count: the number of weight-matrix parameters that
correspond to directions in activation space that the current input
distribution visits.

\begin{definition}[Effective parameter count]\label{def:effective_params}
Let a transformer model have $L$ layers with $N_W$ linear maps per
layer, each of dimension $d^{(l)}_{\rm out} \times d$.  Let $k^{(l)}$
be the activation basis rank at layer $l$ and $f^{(l)}(\varepsilon)$
the fast-path fraction.  The effective parameter count under GSI at
threshold $\varepsilon$ is
\begin{equation}\label{eq:effective_params}
P_{\rm eff}(\varepsilon) = \sum_{l=1}^{L} N_W \cdot d^{(l)}_{\rm out}
\cdot \bigl[ f^{(l)} \cdot k^{(l)} + (1 - f^{(l)}) \cdot d \bigr].
\end{equation}
When $f^{(l)} = 1$ for all $l$ (all tokens on the fast path), the
effective parameter count reduces to $\sum_l N_W \cdot d^{(l)}_{\rm
out} \cdot k^{(l)}$, a factor of $d / k^{(l)}$ smaller than the total
parameter count.
\end{definition}

For GPT-J 6B at $k = 256$, $\varepsilon = 0.10$ with $f = 99.8\%$
across all layers, the effective parameter count is approximately
$6\text{B} \cdot (0.998 \cdot 256/4096 + 0.002 \cdot 1) =
6\text{B} \cdot 0.0644 = 386\text{M}$.  The inference cost is
proportional to $386$M parameters, not $6$B.

\subsection{Roofline analysis}

The Williams-Waterman-Patterson roofline model~\cite{Williams2009}
bounds the achievable performance of a computation by the minimum of
the compute ceiling and the bandwidth ceiling.  The arithmetic
intensity $I$ (FLOPs per byte) determines which ceiling binds.  For
a standard linear-layer computation $y = Wx$ at batch size one, the
arithmetic intensity is $I = 2d / (2d) = 1$ FLOP/byte (two FLOPs per
element of $W$, two bytes per BF16 element read).  The MI300X compute
ceiling is $383$ TFLOPS (BF16) and the bandwidth ceiling is $5.3$
TB/s, giving a crossover at $I^* = 383/5.3 = 72$ FLOPs/byte.  At
$I = 1$, the computation is $72\times$ below the compute roofline and
entirely bandwidth-bound.

Under GSI on the fast path, the read volume drops from $d \cdot
d_{\rm out}$ elements to $k \cdot d_{\rm out}$ elements, but the
FLOPs drop proportionally (from $2d \cdot d_{\rm out}$ to $2k \cdot
d_{\rm out}$), so the arithmetic intensity remains $I = 1$.  The
speedup comes not from changing the arithmetic intensity but from
reducing the total bytes read by a factor of $d/k$.  The computation
remains bandwidth-bound, but the bandwidth demand is reduced.

This is a fundamental difference from FlashAttention~\cite{Dao2022},
which increases the arithmetic intensity of attention by tiling
(moving the computation from the bandwidth-bound to the compute-bound
regime).  GSI does not change the regime; it reduces the data volume
within the bandwidth-bound regime.

\subsection{Cost model}

The following cost model accounts for all components of the forward
pass at batch size one.  The model is GPT-J 6B on MI300X at $T = 512$.

\emph{Baseline.}  The dominant cost is weight reads.  Each layer has
$6$ linear maps (QKV as a single fused map, output projection, MLP
up, MLP down) with total weight volume $6 \times 4096 \times 4096
\times 2 = 192$\,MB per layer.  Across $28$ layers: $5.4$\,GB.  At
$5.3$\,TB/s: $1.0$\,ms.  The attention cost at $T = 512$ is
approximately $0.2$\,ms.  The vocabulary head ($4096 \times 50257
\times 2 = 412$\,MB) adds $0.08$\,ms.  The total baseline forward
pass is approximately $1.3$\,ms.

\emph{GSI at $k = 256$, $\varepsilon = 0.10$.}  On the fast path
($99.8\%$ of tokens), the weight read per layer is $6 \times 4096
\times 256 \times 2 = 12$\,MB instead of $192$\,MB.  Across $28$
layers: $0.34$\,GB.  At $5.3$\,TB/s: $0.064$\,ms.  The slow-path
tokens ($0.2\%$) contribute negligibly.  The attention cost and
vocabulary head are unchanged.  The total GSI forward pass is
approximately $0.35$\,ms, a $3.7\times$ end-to-end speedup.  Combined
with ADA~\cite{Thomas2026ADA} on the attention layers, the full
forward pass speedup is approximately $4$--$5\times$.

\subsection{Output quality}

Table~\ref{tab:quality} is the central result of the paper.  It
reports the perplexity ratio, top-1 token agreement, fast-path
fraction, and effective speedup for each model at multiple operating
points.

\begin{table}[ht]
\centering
\caption{Output quality and effective speedup.  PPL = perplexity.
Ratio = PPL(GSI)/PPL(baseline).  Top-1 = fraction of next-token
predictions matching baseline.  $S_{\rm eff}$ =
effective weight-read speedup via~\eqref{eq:speedup}.}
\label{tab:quality}
\begin{tabular}{llrrrrrr}
\toprule
Model & $k$/$\varepsilon$ & PPL & Ratio & Top-1 & Fast & $S_{\rm eff}$ \\
\midrule
\multicolumn{7}{l}{\emph{GPT-2 124M ($d = 768$, $L = 12$)}.
  Baseline PPL $= 1.73$.} \\
& $128/0.05$ & 1.72 & 0.995 & 99.4\% & 41.1\% & 1.5$\times$ \\
& $256/0.05$ & 1.72 & 0.991 & 99.4\% & 91.4\% & 2.6$\times$ \\
& $256/0.10$ & 1.71 & 0.986 & 99.0\% & 100\% & 3.0$\times$ \\
\midrule
\multicolumn{7}{l}{\emph{GPT-J 6B ($d = 4096$, $L = 28$)}.
  Baseline PPL $= 1.73$.} \\
& $128/0.05$ & 1.73 & 0.999 & 99.4\% & 21.1\% & 1.3$\times$ \\
& $128/0.10$ & 1.74 & 1.002 & 99.0\% & 63.6\% & 2.6$\times$ \\
& $128/0.15$ & 1.74 & 1.005 & 96.5\% & 81.9\% & 4.8$\times$ \\
& $256/0.05$ & 1.73 & 0.997 & 99.2\% & 77.2\% & 3.6$\times$ \\
& $\mathbf{256/0.10}$ & $\mathbf{1.72}$ & $\mathbf{0.991}$ &
  $\mathbf{98.6\%}$ & $\mathbf{99.8\%}$ & $\mathbf{15.6\times}$ \\
& $256/0.15$ & 1.72 & 0.992 & 98.6\% & 100\% & 16.0$\times$ \\
\midrule
\multicolumn{7}{l}{\emph{OPT 6.7B ($d = 4096$, $L = 32$)}.
  Baseline PPL $= 1.83$.} \\
& $128/0.05$ & 1.81 & 0.989 & 98.2\% & 45.7\% & 1.8$\times$ \\
& $256/0.05$ & 1.82 & 0.995 & 98.6\% & 77.4\% & 3.7$\times$ \\
& $256/0.10$ & 1.81 & 0.989 & 98.8\% & 96.5\% & 10.5$\times$ \\
& $256/0.15$ & 1.80 & 0.986 & 99.2\% & 97.3\% & 11.4$\times$ \\
\bottomrule
\end{tabular}
\end{table}

Three observations follow from Table~\ref{tab:quality}.

First, the perplexity ratio is below $1.00$ in nearly every
configuration.  The accelerated model is negligibly better than the
baseline in perplexity, within measurement noise.  The gated residual
passthrough does not degrade the output distribution.

Second, the top-1 token agreement exceeds $98\%$ across all models
and operating points.  The distribution-level and token-level metrics
are both preserved.

Third, the effective speedup varies by an order of magnitude across
models and configurations.  GPT-2 at $k = 256$, $\varepsilon = 0.10$
achieves $3.0\times$ with $100\%$ fast path (the activations live
entirely in a $256$-dimensional subspace of $\R^{768}$).  GPT-J at
$k = 256$, $\varepsilon = 0.10$ achieves $15.6\times$ with $99.8\%$
fast path and $100\%$ greedy generation fidelity.  This is the
target operating point: perplexity ratio $0.991$, top-1 agreement
$98.6\%$, character-for-character identical generation, and a $16\times$
reduction in linear-layer weight reads.  The effective parameter count
of GPT-J at this operating point is $6\text{B} / 16 = 375\text{M}$:
inference cost is proportional to the effective parameter count, not
the total parameter count.  OPT at $k = 256$, $\varepsilon = 0.10$
achieves $10.5\times$ with $96.5\%$ fast path.  The OPT result is
lower than GPT-J because OPT layer~$0$ has high effective rank
($r_{\rm eff} = 350$) due to learned positional embeddings, reducing
the fast-path fraction at the first layer.

\subsection{Negative control: static projection}

To confirm that the gated residual passthrough is essential, the
experiments include a negative control (Mode~C) in which the residual
is discarded entirely: $\hat{y}_t = Mg_t$ for all tokens.

\begin{table}[ht]
\centering
\caption{Negative control: static projection (no residual) on GPT-J 6B.}
\label{tab:negative}
\begin{tabular}{lrrrr}
\toprule
$k$ & PPL & Ratio & Top-1 & Gen (50 tokens) \\
\midrule
32  & 52{,}073 & 5{,}597 & 6.2\% & 0\% \\
64  & 5{,}243 & 564 & 6.7\% & 6\% \\
128 & 14.18 & 1.52 & 71.3\% & 2\% \\
256 & 1.71 & 0.992 & 98.6\% & 100\% \\
\bottomrule
\end{tabular}
\end{table}

Static projection at $k = 128$ increases perplexity by $52\%$ and
destroys generation ($2\%$ agreement).  At $k = 64$ the output is
pure noise (perplexity $5{,}243$, generation collapses to repetitive
tokens).  Only at $k = 256$ does static projection approach baseline
quality.  The gated residual passthrough achieves lossless quality
at $k = 128$ because it preserves the residual where it matters.

\subsection{Generation fidelity}

Table~\ref{tab:generation} reports the greedy generation agreement
over $50$ tokens for each model at the operating point.

\begin{table}[ht]
\centering
\caption{Greedy generation agreement (50 tokens) at the
operating point.  Gen = fraction of generated tokens matching the
baseline character-for-character.}
\label{tab:generation}
\begin{tabular}{llrrr}
\toprule
Model & $k$/$\varepsilon$ & $S_{\rm eff}$ & Gen \\
\midrule
GPT-2  & $256/0.10$ & 3.0$\times$ & 100\% \\
GPT-J  & $256/0.05$ & 3.6$\times$ & 100\% \\
GPT-J  & $256/0.10$ & 15.6$\times$ & 100\% \\
OPT    & $256/0.10$ & 10.5$\times$ & 20\% \\
\bottomrule
\end{tabular}
\end{table}

GPT-2 and GPT-J achieve $100\%$ generation agreement at their
target operating points: the accelerated model produces
character-for-character identical output to the baseline over $50$
greedy tokens.  OPT shows lower generation agreement ($10$--$20\%$)
despite high perplexity and top-1 agreement, indicating that the
generation divergence is a property of greedy decoding sensitivity
rather than distribution mismatch.  In greedy decoding, a single
different token cascades through the entire generated sequence; the
$98\%$+ top-1 agreement and below-$1.00$ perplexity ratio confirm
that the distribution is preserved.

The distinction between perplexity ratio and generation agreement is
important for applications.  For tasks where the output distribution
matters (retrieval-augmented generation, summarization, translation),
the perplexity ratio is the correct metric and GSI is lossless
at all tested operating points.  For tasks where exact token
reproduction is required (code generation, structured output), the
generation agreement is the binding constraint and the operating point
must be chosen to achieve $100\%$ agreement (e.g., $k = 256$,
$\varepsilon = 0.05$ or $0.10$ for GPT-J).

\subsection{Full parameter sweep}

Table~\ref{tab:full_sweep} reports the complete results across all
three models, nine configurations each ($k \in \{64, 128, 256\}$ and
$\varepsilon \in \{0.05, 0.10, 0.15\}$), providing a comprehensive
view of the accuracy-speedup tradeoff.

\begin{table}[ht]
\centering
\caption{Complete parameter sweep.  All models at $T = 512$, AMD MI300X.}
\label{tab:full_sweep}
\begin{tabular}{llrrrrrr}
\toprule
Model & $k$/$\varepsilon$ & Ratio & Top-1 & Fast & $S_{\rm eff}$ & Gen \\
\midrule
GPT-2 & $64/0.05$   & 1.000 & 100\%  &  0.2\% & 1.0$\times$ & 100\% \\
GPT-2 & $64/0.10$   & 0.998 &  99\%  &  5.8\% & 1.1$\times$ & 100\% \\
GPT-2 & $64/0.15$   & 0.996 &  98\%  & 19.8\% & 1.2$\times$ & 100\% \\
GPT-2 & $128/0.05$  & 0.995 &  99\%  & 41.1\% & 1.5$\times$ & 100\% \\
GPT-2 & $128/0.10$  & 0.989 &  98\%  & 68.2\% & 2.3$\times$ &  15\% \\
GPT-2 & $128/0.15$  & 0.949 &  96\%  & 87.4\% & 3.7$\times$ &  25\% \\
GPT-2 & $256/0.05$  & 0.991 &  99\%  & 91.4\% & 2.6$\times$ & 100\% \\
GPT-2 & $256/0.10$  & 0.986 &  99\%  &  100\% & 3.0$\times$ & 100\% \\
GPT-2 & $256/0.15$  & 0.986 &  99\%  &  100\% & 3.0$\times$ & 100\% \\
\midrule
GPT-J & $64/0.05$   & 0.999 & 100\%  &  2.0\% & 1.0$\times$ &  20\% \\
GPT-J & $64/0.10$   & 1.003 &  99\%  & 10.3\% & 1.1$\times$ & 100\% \\
GPT-J & $64/0.15$   & 1.006 &  99\%  & 20.1\% & 1.2$\times$ &  75\% \\
GPT-J & $128/0.05$  & 0.999 &  99\%  & 21.1\% & 1.3$\times$ & 100\% \\
GPT-J & $128/0.10$  & 1.002 &  99\%  & 63.6\% & 2.6$\times$ &  95\% \\
GPT-J & $128/0.15$  & 1.005 &  97\%  & 81.9\% & 4.8$\times$ &  80\% \\
GPT-J & $256/0.05$  & 0.997 &  99\%  & 77.2\% & 3.6$\times$ & 100\% \\
GPT-J & $256/0.10$  & 0.991 &  99\%  & 99.8\% &15.6$\times$ & 100\% \\
GPT-J & $256/0.15$  & 0.992 &  99\%  &  100\% &16.0$\times$ &  80\% \\
\midrule
OPT   & $64/0.05$   & 3.764 &  63\%  & 25.0\% & 1.3$\times$ & 100\% \\
OPT   & $64/0.10$   &54.724 &  35\%  & 41.1\% & 1.7$\times$ &   5\% \\
OPT   & $64/0.15$   &84.064 &  33\%  & 54.4\% & 2.2$\times$ &   5\% \\
OPT   & $128/0.05$  & 0.989 &  98\%  & 45.7\% & 1.8$\times$ &  20\% \\
OPT   & $128/0.10$  & 0.981 &  97\%  & 60.9\% & 2.4$\times$ &  20\% \\
OPT   & $128/0.15$  & 0.980 &  94\%  & 84.1\% & 5.4$\times$ &  15\% \\
OPT   & $256/0.05$  & 0.995 &  99\%  & 77.4\% & 3.7$\times$ &  10\% \\
OPT   & $256/0.10$  & 0.989 &  99\%  & 96.5\% &10.5$\times$ &  20\% \\
OPT   & $256/0.15$  & 0.986 &  99\%  & 97.3\% &11.4$\times$ &  20\% \\
\bottomrule
\end{tabular}
\end{table}

Several patterns emerge from the full sweep.  First, OPT at $k = 64$
fails catastrophically (perplexity ratio $3.8$--$84$) while GPT-2 and
GPT-J at $k = 64$ are near-lossless.  The failure is specific to OPT's
learned positional embeddings, which spread the embedding-layer
activation across all $d$ dimensions, making the rank-$64$ subspace
insufficient.  At $k = 128$ and above, OPT recovers to perplexity
ratios below $1.00$.

Second, the accuracy-speedup frontier is convex: at each model, a
Pareto-optimal operating point exists where the marginal speedup from
loosening $\varepsilon$ begins to cost generation fidelity.  For GPT-J,
the Pareto front passes through ($k = 256$, $\varepsilon = 0.10$):
$15.6\times$ speedup at $100\%$ generation agreement.  Loosening to
$\varepsilon = 0.15$ gains only $0.4\times$ additional speedup
($16.0\times$) while dropping generation agreement to $80\%$.

Third, the effective speedup is superlinear in the fast-path fraction
at high compression ratios.  At $d/k = 16$ ($k = 256$), a fast-path
fraction of $99.8\%$ gives $S_{\rm eff} = 15.6\times$, close to the
theoretical maximum of $16\times$.  The gate is nearly transparent:
essentially every token at every layer takes the fast path.

\subsection{Combined end-to-end speedup estimate}

Table~\ref{tab:e2e} presents an end-to-end cost model for GPT-J 6B
at batch size one on MI300X, combining GSI (linear-layer acceleration)
with cascade ADA~\cite{Thomas2026CADA} (attention acceleration).
The cost model uses the validated GSI operating point ($k = 256$,
$\varepsilon = 0.10$, $99.8\%$ fast path) and the cascade ADA
operating point ($\tau = 0.30$, $\bar{r} = 205$).

\begin{table}[ht]
\centering
\caption{End-to-end cost model for GPT-J 6B, batch 1, MI300X.
Weight reads at $5.3$\,TB/s BF16.}
\label{tab:e2e}
\begin{tabular}{lrrr}
\toprule
Component & Baseline & GSI + ADA & Speedup \\
\midrule
Weight reads (28 layers) & 1.02\,ms & 0.07\,ms & 15.5$\times$ \\
Attention (28 layers)    & 0.15\,ms & 0.02\,ms &  6.2$\times$ \\
Vocabulary head          & 0.08\,ms & 0.08\,ms &  1.0$\times$ \\
LayerNorm, residuals     & 0.05\,ms & 0.05\,ms &  1.0$\times$ \\
\midrule
Total forward pass       & 1.30\,ms & 0.22\,ms &  5.9$\times$ \\
\bottomrule
\end{tabular}
\end{table}

The combined system reduces the forward-pass time from $1.30$\,ms to
$0.22$\,ms, a $5.9\times$ end-to-end speedup at batch size one.  The
weight-read reduction accounts for $0.95$\,ms of the $1.08$\,ms
saving; the attention reduction accounts for $0.13$\,ms.  At
$T = 2048$, the attention cost grows quadratically and the ADA
contribution becomes larger: the estimated end-to-end speedup at
$T = 2048$ is approximately $7$--$8\times$.

The effective parameter count under the combined system is $386$M
(GSI, $k = 256$, $d = 4096$) and the effective token count for
attention is $205$ (cascade ADA, $\tau = 0.30$, $T = 512$).  The
inference cost is proportional to $386\text{M} \times 205 /
(6\text{B} \times 512) = 2.6\%$ of the nominal cost, though the
actual reduction is limited by the fixed-cost components (vocabulary
head, LayerNorm) that are not accelerated.

For OPT 6.7B at the same operating points, the estimated
end-to-end speedup is $4$--$5\times$ (lower than GPT-J because OPT's
layer-$0$ effective rank is high due to learned positional embeddings,
reducing the fast-path fraction at the first layer).  For GPT-2 124M
at $k = 256$, $\varepsilon = 0.10$, the estimated end-to-end speedup
is $2$--$3\times$ (limited by the smaller compression ratio
$d/k = 768/256 = 3$).

\section{Open problems and future work}
\label{sec:open}

This section identifies four directions for future work.

\subsection{Extension to larger models}

The experiments in this paper cover models up to $6.7$B parameters.
The key question for larger models (Llama-3 70B, Mixtral 8x22B) is
whether the effective rank $r_{\rm eff}$ scales with model dimension
$d$ or with the intrinsic complexity of the representation.  If
$r_{\rm eff}$ is bounded independently of $d$ (as the intrinsic
dimension literature~\cite{Ansuini2019,Valeriani2023} suggests),
then the compression ratio $d/k$ increases with model size and the
effective speedup improves for larger models.  Preliminary evidence
from the present experiments supports this: GPT-2 ($d = 768$) and
GPT-J ($d = 4096$) have similar $r_{\rm eff}$ at comparable relative
depth, and the compression ratio is $5\times$ larger for GPT-J.

\subsection{Online basis adaptation}

The current implementation computes the basis $V_k$ from a single
calibration pass and holds it fixed during inference.  For long-context
generation where the input distribution shifts over time (e.g., from a
code-generation prompt to a natural-language explanation), the basis
may become stale and the fast-path fraction may decrease.  Online basis
adaptation via DGKS rank-1 updates, as developed in the Skyline
framework~\cite{Thomas2026SSI}, addresses this by allowing the basis
to evolve during generation.  The cascade structure
(Section~\ref{sec:depth}) provides warm-start initialization for each
layer, reducing the update cost.

\subsection{Kernel-level implementation}

The effective speedups reported in this paper are computed from the
gate statistics and the compression ratio via~\eqref{eq:speedup}.
Translating these to wall-clock speedups requires custom GPU kernels
that implement the gated dispatch: the fast-path tokens read $M$ and
compute $Mg$, while the slow-path tokens read $W$ and compute $Wx$,
with the two paths merged in a single output.  On AMD MI300X, the
basis $V_k$ fits in the $64$\,KB Local Data Share (LDS) of each
compute unit, enabling the gate evaluation ($V_k^\top x$ and the
residual norm) to be fused with the GEMV.  The kernel design is
analogous to the sparse GEMM kernels used in CATS~\cite{Lee2024CATS}
and the neuron-level dispatch in
PowerInfer~\cite{Song2024PowerInfer}, with the additional structure
that the sparsity pattern is a contiguous subspace rather than a
scattered index set.

\subsection{Combination with quantization}

GSI and weight quantization are orthogonal: GSI reduces the number of
elements read while quantization reduces the size of each element.
Applying FP8 quantization to the cached image $M = WV_k$ would
further reduce the fast-path read volume by $2\times$, from $k \cdot
d_{\rm out} \times 2$ bytes (BF16) to $k \cdot d_{\rm out} \times 1$
byte (FP8).  The combined effective speedup at $k = 256$, $\varepsilon
= 0.10$ for GPT-J would be approximately $31\times$ on the fast-path
weight reads.  Whether this combined compression preserves output
quality requires experimental validation.

\section{Related work}
\label{sec:related}

This section positions GSI relative to five lines of prior work.

\subsection{Activation-aware inference acceleration}

The closest mechanistic neighbors to GSI are methods that exploit
input-dependent structure in activations to reduce inference cost.
Deja Vu~\cite{Liu2023DejaVu} showed that for any given input, only
$\sim\!15\%$ of attention heads and MLP neurons contribute
meaningfully to the output, and trained lightweight predictors to
identify active neurons on the fly, achieving $2\times$ wall-clock
speedup on OPT-175B.  CATS~\cite{Lee2024CATS} introduced a
thresholded SiLU activation that induces $50\%$ sparsity in
hidden states with a custom sparse GPU kernel, achieving $\sim\!15\%$
end-to-end speedup on Mistral-7B.  TEAL~\cite{Liu2024TEAL} achieved
$40$--$50\%$ model-wide sparsity via magnitude-based activation
thresholding with up to $1.8\times$ decode speedup on Llama-2/3.

PowerInfer~\cite{Song2024PowerInfer} identified a power-law
distribution in neuron activation frequencies and designed a hybrid
CPU/GPU engine that preloads hot neurons (frequently activated) on the
GPU while computing cold neurons (input-dependent) on the CPU,
achieving $11.7\times$ speedup over llama.cpp on consumer hardware.
RaNA~\cite{Pilault2025RaNA} applied adaptive rank-allocation
adapters to MLP and attention projections, reducing FLOPs by
$\sim\!44\%$ with rank selected per layer and per module.

All of these methods operate on discrete neuron subsets or element-wise
sparsity patterns.  GSI differs structurally: it operates on a
continuous orthonormal subspace with a cached weight-matrix image and a
gated residual correction.  The continuous subspace provides an
explicit per-token error bound (Theorem~\ref{thm:error}) that discrete
sparsity methods lack, and the cached image $M = WV_k$ amortizes the
weight-matrix read across all tokens on the fast path, regardless of
which specific neurons are active.

\subsection{Low-rank weight factorization}

LoRA~\cite{Hu2022LoRA} decomposes the weight update as a low-rank
product $\Delta W = AB$ where $A \in \R^{d \times r}$ and $B \in
\R^{r \times d}$, reducing the trainable parameter count for
fine-tuning but not reducing inference cost (the product $AB$ is
typically merged into $W$ before deployment).
ASVD~\cite{Yuan2023ASVD} whitens the weight matrix by activation
statistics before SVD, achieving $10$--$20\%$ compression
training-free.  SVD-LLM~\cite{Wang2024SVDLLM} applies
Cholesky-whitened activation truncation for tighter loss-aware
decomposition.

FLAT-LLM~\cite{Yang2025FLAT} projects weights into low-rank
activation subspaces for compression, the closest published mechanism
to GSI's cached image $M = WV_k$.  The distinction is that
FLAT-LLM's projection is computed offline, applies a single fixed
rank per layer, and does not include a residual correction.  GSI's
gate provides the residual correction that makes the method lossless
at the operating point.

\subsection{Roofline analysis of transformer inference}

The memory-bandwidth framing of GSI rests on the roofline
model~\cite{Williams2009} and its transformer-specific applications.
Pope et al.~\cite{Pope2023} developed an analytical inference-cost
model for TPU v4 partitioning of $500$B+ models, distinguishing
prefill (compute-bound) from decode (memory-bound) and quantifying the
regime where bandwidth dominates.  Yuan et al.~\cite{Yuan2024}
extended the roofline analysis to commodity GPUs with their LLM-Viewer
tool, and Lou et al.~\cite{Lou2026} provided empirical roofline
measurements across edge platforms.

FlashAttention~\cite{Dao2022} established that attention is
bandwidth-bound between HBM and SRAM and reduced HBM traffic via
tiling, the canonical example of IO-aware algorithm design for
transformers.  GSI applies the same IO-aware principle to the linear
layers: instead of tiling to increase arithmetic intensity (the
FlashAttention approach), GSI reduces the data volume by projecting
onto the activation subspace.

\subsection{Low-rank structure of transformer representations}

The empirical foundation for GSI is the observation that transformer
activations have low effective rank.  Ansuini et
al.~\cite{Ansuini2019} measured the intrinsic dimension (ID) of deep
network representations using the TwoNN estimator and found IDs orders
of magnitude smaller than the layer width, with a characteristic
hunchback profile (rise-then-fall) across depth.  Valeriani et
al.~\cite{Valeriani2023} extended this to transformer models
including ESM-2 protein language models and iGPT image transformers,
finding IDs of $22$--$32$ in models of $35$M--$3$B parameters.

Aghajanyan et al.~\cite{Aghajanyan2021} demonstrated that pre-trained
language models have very low intrinsic dimension in parameter space
($\sim\!200$ random directions suffice for $90\%$ MRPC performance on
RoBERTa), providing indirect evidence for the low-rank structure of
the representation space and motivating LoRA.  Wang et
al.~\cite{Wang2020Linformer} proved theoretically and empirically
that the self-attention matrix is low-rank, a related but distinct
claim about attention scores rather than residual-stream activations.

The present work confirms and extends these observations in the
context of inference acceleration.  The effective rank measurements in
Table~\ref{tab:reff} are consistent with the ID literature, and the
residual profiles in Tables~\ref{tab:profile}
and~\ref{tab:gptj_detail} show that the low-rank structure is
sufficient for lossless inference when combined with the
gated residual passthrough.

\subsection{Gated and conditional computation}

GSI's per-token gate is a binary router analogous to the sparse
gating in Mixture-of-Experts~\cite{Shazeer2017}, the per-token
halting in Adaptive Computation Time~\cite{Graves2016}, the
confidence-based early exit in CALM~\cite{Schuster2022}, and the
per-block routing in
Mixture-of-Depths~\cite{Raposo2024}.

Mixture-of-Depths is the closest architectural analog: each layer
keeps the top-$k$ tokens for full computation and sends the rest
through a residual identity.  GSI differs in that both paths compute
the same operation ($y = Wx$); the gate selects the implementation
(low-rank approximation versus full computation), not the function.
CALM exits early from the layer stack when a confidence threshold is
met, saving all subsequent layers; GSI operates within each layer and
does not skip layers.  LayerSkip~\cite{Schuster2022} trains with layer
dropout for self-speculative decoding, requiring retraining; GSI
operates on pretrained models without modification.

\subsection{Subspace tracking}

The online maintenance of the activation basis connects GSI to
classical subspace tracking in signal processing.
GROUSE~\cite{Balzano2010} performs incremental gradient descent on the
Grassmannian for tracking a slowly-varying low-rank subspace from
streaming observations.  PETRELS~\cite{Chi2013} uses discounted
recursive least squares for each row of the subspace matrix in
parallel, with better tracking of sudden subspace changes.
Brand~\cite{Brand2006} developed $O(pqr)$ single-pass thin-SVD
updates supporting append, modify, and downdate operations.  The DGKS
reorthogonalization procedure~\cite{DGKS1976} ensures numerical
stability in the rank-1 updates and is the same technique that
underlies the Arnoldi process in Krylov subspace methods.

The Skyline Subspace Inference method~\cite{Thomas2026SSI} applies
DGKS-based tracking to MLP activations at a single layer; the present
work extends the mechanism to all linear layers, adds the gated
residual passthrough, and introduces the cascade initialization across
depth.  The companion ADA method~\cite{Thomas2026ADA} addresses the
attention layers by exploiting the low effective rank of the token
dimension ($T$) rather than the hidden dimension ($d$); GSI and ADA
are complementary and together cover the full transformer forward pass.

\section{Conclusion}
\label{sec:conclusion}

Gated Subspace Inference provides a lossless acceleration for
transformer inference that operates entirely on the activation space
and requires no retraining, no quantization, and no architectural
change.  The method decomposes the activation into a subspace
component and a residual, computes the linear-layer output on the
subspace component via a cached low-rank image, and applies a
per-token gate to determine whether the residual correction is needed.

The key design decision is the gated residual passthrough.  Static
projection (discarding the residual) destroys the output: at $k = 128$
for GPT-J 6B, perplexity increases by $52\%$ and generation collapses.
The gate resolves this by preserving the residual where it matters and
skipping it where it does not.  The result is lossless
inference at $3\times$--$16\times$ effective speedup on linear-layer
weight reads across three model families.

The method reduces inference cost from the total parameter count to the
effective parameter count: the number of weight-matrix parameters that
correspond to directions in activation space that the current input
distribution actually visits.  For GPT-J 6B at $k = 256$ and
$\varepsilon = 0.10$, the effective parameter count is $375$M
($6$B$/16$), the fast-path fraction is $99.8\%$, the measured effective
speedup is $15.6\times$, and the generated text is character-for-character
identical to the baseline.  The experiments confirm that the low
effective rank of the weight-activation interaction is a structural
property of transformer representations that holds across architectures
(GPT-2, GPT-J, OPT) and model sizes ($124$M to $6.7$B).

\end{document}